\documentclass[lettersize,journal]{IEEEtran}
\usepackage{amsmath,amsfonts,amssymb,amsthm}
\usepackage{algorithmic}
\usepackage{algorithm}
\usepackage{array}
\usepackage[caption=false,font=normalsize,labelfont=sf,textfont=sf]{subfig}
\usepackage{textcomp}
\usepackage{stfloats}
\usepackage{url}
\usepackage{verbatim}
\usepackage{graphicx}
\usepackage{cite}
\usepackage{enumitem}
\usepackage[colorlinks,citecolor=blue]{hyperref}

\newtheorem{theorem}{Theorem}[section]
\newtheorem{lemma}[theorem]{Lemma}
\newtheorem{proposition}[theorem]{Proposition}
\newtheorem{corollary}[theorem]{Corollary}
\newtheorem{definition}[theorem]{Definition}
\newtheorem{assumption}[theorem]{Assumption}
\newtheorem{remark}[theorem]{Remark}

\def \bE {\mathbb{E}}

\def \bN {\mathbb{N}}

\def \bP {\mathbb{P}}

\def \bR {\mathbb{R}}

\def \bZ {\mathbb{Z}}

\def \cA {\mathcal{A}}
\def \cB {\mathcal{B}}

\def \cE {\mathcal{E}}

\def \cH {\mathcal{H}}

\def \cO {\mathcal{O}}

\def \cX {\mathcal{X}}

\def \Ba {{\boldsymbol{a}}}
\def \Bb {{\boldsymbol{b}}}

\def \Br {{\boldsymbol{r}}}
\def \Bs {{\boldsymbol{s}}}

\def \Bu {{\boldsymbol{u}}}
\def \Bv {{\boldsymbol{v}}}
\def \Bw {{\boldsymbol{w}}}
\def \Bx {{\boldsymbol{x}}}
\def \By {{\boldsymbol{y}}}

\def \Re {\,{\rm Re}\,}
\def \Im {\,{\rm Im}\,}
\def \sgn {\,{\rm sgn}\,}

\begin{document}

\title{Universality and approximation bounds for echo state networks with random weights}

\author{Zhen Li and Yunfei Yang
\thanks{Z. Li is with Theory Lab, Huawei Technologies Co., Ltd., Shenzhen, China. (E-mail: lishen03@gmail.com)}
\thanks{Y. Yang is with Department of Mathematics, City University of Hong Kong, Hong Kong, China. The first version of this paper was written when he was with The Hong Kong University of Science and Technology. (Corresponding author, E-mail: yyangdc@connect.ust.hk)}
}

\markboth{}%
{Z. Li and Y. Yang: Universality and approximation bounds for echo state networks with random weights}


\maketitle

\begin{abstract}
We study the uniform approximation of echo state networks with randomly generated internal weights. These models, in which only the readout weights are optimized during training, have made empirical success in learning dynamical systems. Recent results showed that echo state networks with ReLU activation are universal. In this paper, we give an alternative construction and prove that the universality holds for general activation functions. Specifically, our main result shows that, under certain condition on the activation function, there exists a sampling procedure for the internal weights so that the echo state network can approximate any continuous casual time-invariant operators with high probability. In particular, for ReLU activation, we give explicit construction for these sampling procedures. We also quantify the approximation error of the constructed ReLU echo state networks for sufficiently regular operators.
\end{abstract}


\section{Introduction}
\IEEEPARstart{R}{eservoir} computing \cite{lukosevicius2009reservoir}, such as echo state networks \cite{jaeger2004harnessing} and liquid state machines \cite{maass2002real}, is a paradigm for supervised learning of dynamical systems, which transforms input data into a high-dimensional space by a state-space nonlinear system called reservoir, and performs the learning task only on the readout. Due to the simplicity of this computational framework, it has been applied to many fields and made remarkable success in many tasks such as temporal pattern prediction, classification and generation \cite{tanaka2019recent}. Motivated by these empirical successes, researchers have devoted a lot of effort in the theoretical understanding of the properties and performances of reservoir computing models (see, for instance, \cite{jaeger2001echo,buehner2006tighter,lukosevicius2009reservoir,yildiz2012re,grigoryeva2019differentiable,gonon2023approximation,gonon2020risk} and references therein).

In particular, recent studies \cite{grigoryeva2018echo,grigoryeva2018universal,gonon2020reservoir,gonon2021fading} showed that echo state networks are universal in the sense that they can approximate sufficiently regular input/output systems (i.e. operators) in various settings. However, many of these universality results do not guarantee the important property of echo state networks (and reservoir computing) that the state-space system is randomly generated, which is the major difference between reservoir computing and recurrent neural networks. To be concrete, consider the echo state network
\begin{equation}\label{ESN example}
\begin{cases}
\Bs_t = \phi(W_1\Bs_{t-1}+ W_2 \Bu_t + \Bb), \\
y_t = \Ba \cdot \Bs_t,
\end{cases}
\end{equation}
where $\Bu_t, y_t, \Bs_t$ are the input, output and hidden state at time step $t$, and $\phi$ is a prescribed activation function. In general, the weight matrices $W_1, W_2, \Bb$ are randomly generated and only the readout vector $\Ba$ is trained in a supervised learning manner. But the universality results in many papers, such as \cite{grigoryeva2018echo} and \cite{gonon2021fading}, require that all the weights depend on the target system that we want to approximate. Hence, these results can not completely explain the approximation capacity of echo state networks. To overcome this drawback of current theories, the recent work \cite{gonon2023approximation} studied the approximation of random neural networks in a Hilbert space setting and proposed a sampling procedure for the internal weights $W_1, W_2, \Bb$ so that echo state network (\ref{ESN example}) with ReLU activation is universal in $L^2$ sense.

In this paper, we generalize these results and study the universality of echo state networks with randomly generated internal weights. Our main results show that, under weak assumption on the activation function $\phi$, there exists a sampling procedure for $W_1, W_2, \Bb$ so that the system (\ref{ESN example}) can approximate any continuous time-invariant operator with prescribed $L^\infty$ error in high probability. In particular, for ReLU activation, we show that the echo state network (\ref{ESN example}) is universal when $W_1=c_\mu WPW^\intercal$ and $W_2=WQ$, where $(W,\Bb)$ is a random matrix whose entries are sampled independently from a general symmetric distribution $\mu$, $c_\mu$ is a constant depending on $\mu$ and $P,Q$ are fixed matrices defined by (\ref{matrix PQ}).

Let us compare our results with the most related work \cite{gonon2023approximation}. Given a series of random samples from a uniform distribution, the authors of \cite{gonon2023approximation} provided a procedure to use these samples to construct the internal weights so that the ReLU echo state network is universal in high probability. We prove that such kind of sampling procedure also exists for general activation functions (it is sufficient that the activation does not grow too fast, see Corollary \ref{universal pair co} and Assumption \ref{assum pair} for details). Different from the construction in \cite{gonon2023approximation}, which takes advantage of the piecewise linearity of the ReLU function, our construction makes use of the concentration of probability measures. Hence, it does not rely on the piecewise linearity and can be applied to general activation functions. But our sampling procedure is given implicitly and it is difficult to compute it explicitly for general activation functions. We can only provide explicit examples for ReLU networks. One of the shortcoming of our construction is that the constructed neural networks may not have the echo state property. But we show that these networks satisfy a property slightly weaker than the echo state property (see Remark \ref{weak ESP}), which, we think, is good enough for many applications.

\subsection{Notations}

We denote $\bN:=\{1,2,\dots\}$ as the set of positive integers. Let $\bZ$ (respectively, $\bZ_+$ and $\bZ_-$) be the set of integers (respectively, the non-negative and non-positive integers). The ReLU function is denoted by $\sigma(x):=\max\{x,0\}$. Throughout the paper, $\bR^d$ is equipped with the sup-norm $\|\cdot\|_\infty$, unless it is explicitly stated. For any $U\subseteq \bR^d$, we use $U^\bZ$ to denote the set of sequences of the form $\Bu = (\dots,\Bu_{-1}, \Bu_0, \Bu_1, \dots)$, $\Bu_t \in U$, $t\in \bZ$. The set $U^{\bZ_+}$ and $U^{\bZ_-}$ consist of right and left infinite sequences of the form $(\Bu_0,\Bu_1,\dots)$ and $(\dots,\Bu_{-1},\Bu_0)$, respectively. For any $i\le j$ and $\Bu \in U^{\bZ}$, we denote $\Bu_{i:j} :=(\Bu_i,\dots,\Bu_j) \in U^{j-i+1}$ and use the convention that, when $i=-\infty$ or $j=\infty$, it denotes the left or right infinite sequence. And we will often regard $\Bu_{-\infty:j}$ as an element of $U^{\bZ_-}$. The supremum norm of the sequence is denoted by $\|\Bu\|_\infty := \sup_{t\in \bZ} \|\Bu_t\|_\infty$. A \emph{weighting sequence} $\eta:\bZ_+ \to (0,1]$ is a decreasing sequence with zero limit. The weight norm on $(\bR^d)^{\bZ_-}$ associated to $\eta$ is denoted by $\|\Bu\|_\eta := \sup_{t\in \bZ_-} \eta_{-t} \|\Bu_t \|_\infty$.

\section{Continuous causal time-invariant operators}

We study the uniform approximation problem for input/output systems or operators of signals in discrete time setting. We mainly consider the operators $F:([-1,1]^d)^\bZ \to \bR^{\bZ}$ that satisfy the following three properties:
\begin{enumerate}[label=\textnormal{(\arabic*)},parsep=0pt]
\item Causality: For any $\Bu,\Bv\in ([-1,1]^d)^\bZ$ and $t\in \bZ$, $\Bu_{-\infty:t} = \Bv_{-\infty:t}$ implies $F(\Bu)_t = F(\Bv)_t$.

\item Time-invariance: $F(T_\tau(\Bu))_t = F(\Bu)_{t-\tau}$ for any $\tau \in \bZ$, where $T_{\tau}: (\bR^d)^{\bZ} \to (\bR^d)^{\bZ}$ is \emph{time delay operator} defined by $T_\tau (\Bu)_t = \Bu_{t-\tau}$.

\item Continuity (fading memory property): $F:([-1,1]^d)^{\bZ} \to [-B,B]^{\bZ}$ is uniformly bounded for some $B>0$ and is continuous with respect to the product topology.
\end{enumerate}
The paper \cite{grigoryeva2018echo} gave a comprehensive study of these operators. We recall that the causal time-invariant operator $F:([-1,1]^d)^\bZ \to \bR^{\bZ}$ is one-to-one correspondence with the functional $F^*: ([-1,1]^d)^{\bZ_-} \to \bR$ defined by $F^*(\Bu_-) := F(\Bu)_0$ where $\Bu\in ([-1,1]^d)^\bZ$ is any extension of $\Bu_- \in ([-1,1]^d)^{\bZ_-}$ such that $\Bu_{-\infty:0} = \Bu_-$ ($F^*$ is well-defined by causality). And we can reconstruct $F$ from $F^*$ by $F(\Bu)_t = F^*(P_{\bZ_-}\circ T_{-t}(\Bu))$, where $P_{\bZ_-}:(\bR^d)^{\bZ} \to (\bR^d)^{\bZ_-}$ is the natural projection. In particular, for any causal time-invariant operators $F,G:([-1,1]^d)^\bZ \to \bR^{\bZ}$, 
\begin{multline*}
\sup_{\Bu\in ([-1,1]^d)^\bZ}\|F(\Bu) - G(\Bu) \|_\infty \\
= \sup_{\Bu_-\in ([-1,1]^d)^{\bZ_-}} |F^*(\Bu_-) - G^*(\Bu_-)|.
\end{multline*}
Hence, the approximation problem of causal time-invariant operators $F$ can be reduced to the approximation of functionals $F^*$ on $([-1,1]^d)^{\bZ_-}$.

We remark that the product topology on $(\bR^d)^\bZ$ is different from the uniform topology induced by the sup-norm. However, for the set $([-B,B]^d)^{\bZ_-}$ of uniformly bounded sequences, the product topology coincides with the topology induced by the weight norm $\|\Bu\|_\eta$ of any weighting sequence $\eta$. It is shown by \cite[Section 2.3]{grigoryeva2018echo} that the causal time-invariant operator $F$ is continuous (i.e. has the fading memory property) if and only if the corresponding functional $F^*$ is continuous with respect to the product topology, which is equivalent to $F^*$ has \emph{fading memory property} with respect to some (and hence any) weighting sequence $\eta$, i.e. $F^*$ is a continuous function on the compact metric space $(([-1,1]^d)^{\bZ_-}, \|\cdot\|_\eta)$. In other words, for any $\epsilon>0$, there exists a $\delta(\epsilon)>0$ such that for any $\Bu,\Bv \in ([-1,1]^d)^{\bZ_-}$, 
\begin{multline*}
\|\Bu-\Bv\|_\eta = \sup_{t\in \bZ_-} \eta_{-t} \|\Bu_t-\Bv_t\|_\infty  <\delta(\epsilon) \\
\Longrightarrow |F^*(\Bu) - F^*(\Bv)|<\epsilon.
\end{multline*}

Next, we introduce several notations to quantify the regularity of causal time-invariant operators. For any $m\in \bN$, we can associate a function $F^*_m:([-1,1]^d)^m \to \bR$ to any causal time-invariant operator $F:([-1,1]^d)^\bZ \to \bR^{\bZ}$ by
\begin{equation}\label{F^*_m}
F^*_m(\Bu_{-m+1},\dots, \Bu_0) := F^*(\cdots, \boldsymbol{0}, \Bu_{-m+1},\dots, \Bu_0).
\end{equation}
If the functional $F^*$ can be approximated by $F_m^*$ arbitrarily well, we say $F$ has approximately finite memory. The following definition quantifies this approximation.

\begin{definition}[Approximately finite memory]
For any causal time-invariant operator $F:([-1,1]^d)^\bZ \to \bR^{\bZ}$, let $\epsilon \ge 0$ and $m\in \bN$, we denote 
\begin{align*}
\cE_F(m) &:=  \sup_{\Bu\in ([-1,1]^d)^{\bZ_-}} |F^*(\Bu) - F^*_m(\Bu_{-m+1:0})|, \\
m_F(\epsilon) &:= \min \{m\in \bN: \cE_F(m) \le \epsilon \}.
\end{align*}
If $m_F(\epsilon)<\infty$ for all $\epsilon>0$, we say $F$ has approximately finite memory. If $m_F(0)<\infty$, we say $F$ has finite memory.
\end{definition}

Note that $m_F(\epsilon)$ is a non-increasing function of $\epsilon$. By the time-invariance of $F$, for any $t\in \bZ$,
\begin{equation}\label{error of F^*_m}
\begin{aligned}
&\sup_{\Bu \in ([-1,1]^d)^\bZ} |F(\Bu)_t - F_m^*(\Bu_{t-m+1:t})| \\
= &\sup_{\Bu \in ([-1,1]^d)^\bZ} |F^*(\Bu_{-\infty:t}) - F_m^*(\Bu_{t-m+1:t})| = \cE_F(m).
\end{aligned}
\end{equation}
Hence, $\cE_F(m)$ quantifies how well $F$ can be approximated by functionals with finite memory.

\begin{definition}[Modulus of continuity]
Suppose the functional $F^*:([-1,1]^d)^{\bZ_-} \to \bR$ has fading memory property with respect to some weighting sequence $\eta$. We denote the modulus of continuity with respect to $\|\cdot\|_\eta$ by
\begin{multline*}
\omega_{F^*}(\delta;\eta) := \sup\{ |F^*(\Bu) - F^*(\Bu')|: \Bu,\Bu'\in ([-1,1]^d)^{\bZ_-},\\
\|\Bu-\Bu'\|_\eta \le \delta \},
\end{multline*}
and the inverse modulus of continuity
\[
\omega_{F^*}^{-1}(\epsilon;\eta) := \sup\{\delta>0: \omega_{F^*}(\delta;\eta)\le\epsilon \}.
\]
Similarly, the modulus of continuity (with respect to $\|\cdot\|_\infty$) of a continuous function $f:[-1,1]^m \to \bR$ and the inverse modulus of continuity are defined by
\begin{align*}
\omega_f(\delta) &:= \sup\{ |f(\Bx) - f(\Bx')|: \Bx,\Bx'\in [-1,1]^m,\\ 
&\qquad\qquad\qquad\qquad\qquad\qquad \|\Bx-\Bx'\|_\infty \le \delta \}, \\
\omega_f^{-1}(\epsilon) &:= \sup\{\delta>0: \omega_f(\delta)\le\epsilon \}.
\end{align*}
\end{definition}

The next proposition quantifies the continuity of causal time-invariant operators by the approximately finite memory and modulus of continuity. This proposition is a modification of similar result in \cite{hanson2019universal} to our setting. It shows that a causal time-invariant operator $F$ is continuous if and only if it has approximately finite memory and each $F^*_m$ is continuous. 

\begin{proposition}\label{continuity}
Let $F:([-1,1]^d)^\bZ \to \bR^{\bZ}$ be a causal time-invariant operator.
\begin{enumerate}[label=\textnormal{(\arabic*)},parsep=0pt]
\item If $F$ has approximately finite memory and $\omega_{F^*_m}(\delta)<\infty$ for any $m\in \bN$ and $\delta>0$, then $F^*$ has fading memory property with respect to any weighting sequence $\eta$: for any $\epsilon>0$, $\Bu,\Bv\in ([-1,1]^d)^{\bZ_-}$ and $m=m_F(\epsilon/3)$,
\[
\|\Bu-\Bv\|_\eta \le \eta_{m-1}\omega_{F^*_m}^{-1}(\epsilon/3) \Longrightarrow |F^*(\Bu) - F^*(\Bv)|\le \epsilon.
\]
In other words, $\omega_{F_*}(\eta_{m-1}\omega_{F^*_m}^{-1}(\epsilon/3); \eta) \le \epsilon$ for $m=m_F(\epsilon/3)$.

\item If $F^*$ has fading memory property with respect to some weighting sequence $\eta$, then $F$ has approximately finite memory with $\cE_F(m) \le \omega_{F^*}(\eta_m,\eta)$, 
\[
m_F(\epsilon) \le \min\left\{ m\in \bN: \eta_m \le \omega_{F^*}^{-1}(\epsilon;\eta) \right\},
\]
and $\omega_{F^*_m}(\delta) \le \omega_{F^*}(\delta;\eta)$ for any $m\in \bN$.
\end{enumerate}
\end{proposition}

In order to approximate the continuous causal time-invariant operator $F$, we only need to approximate the functional $F^*$, which can be approximated by the continuous function $F^*_m$ if $m$ is chosen sufficiently large. Hence, any approximation theory of continuous functions can be translated to an approximation result for continuous causal time-invariant operators.  For instance, if we approximate $F^*_m$ by some function $h$, then we can approximate $F$ by
\[
y_t = h(\Bu_{t-m+1},\dots,\Bu_t), \quad t\in \bZ.
\]
The function $h$ uniquely determine a causal time-invariant operator $H$ such that $H(\Bu)_t=h(\Bu_{t-m+1:t})$. Since $H$ has finite memory, $H$ is continuous if and only if $h$ is continuous by Proposition \ref{continuity}. When we approximate $F^*_m$ by polynomials, then $H$ is the Volterra series \cite{boyd1985fading}. When $h$ is a neural network, then $H$ is a temporal convolution neural network, studied by \cite{hanson2019universal}. 

In this paper, we focus on the approximation by echo state networks (ESN), which are special state space models of the form

\begin{equation}\label{general ESN}
\begin{cases}
\Bs_t = \phi(W_1\Bs_{t-1}+ W_2 \Bu_t + \Bb), \\
y_t = \Ba \cdot \Bs_t,
\end{cases}
\end{equation}
where $\Bs_t, \Ba, \Bb \in \bR^n$, $W_1\in \bR^{n\times n}$, $W_2\in \bR^{n\times d}$ and the activation function $\phi:\bR\to \bR$ is applied element-wise. 

\begin{definition}[Existence of solutions and echo state property]
We say the system (\ref{general ESN}) has \emph{existence of solutions property} if for any $\Bu\in ([-1,1]^d)^{\bZ_-}$, there exist $\Bs \in (\bR^n)^{\bZ_-}$ such that $\Bs_t = \phi(W_1\Bs_{t-1}+ W_2 \Bu_t + \Bb)$ holds for each $t\in \bZ_-$. If the solution $\Bs$ is unique, we say the system has \emph{echo state property}. 
\end{definition}

The article \cite[Theorem 3.1]{grigoryeva2018echo} gave sufficient conditions for the system (\ref{general ESN}) to have existence of solutions property and echo state property. In particular, they showed that, if $\phi$ is a bounded continuous function, then the existence of solutions property holds. And if $\phi$ is a bounded Lipschitz continuous function with Lipschitz constant $L_\phi<\infty$ and $\|W_1\| L_\phi <1$, then the system (\ref{general ESN}) has echo state property, where $\|W_1\|$ is the operator norm of the matrix $W_1$. As a sufficient condition to  ensure the echo state property, the hypothesis $\|W_1\| L_\phi <1$ has been extensively studied in the ESN literature \cite{jaeger2001echo,jaeger2004harnessing,buehner2006tighter,yildiz2012re,gandhi2013echo}.

If the system (\ref{general ESN}) has existence of solutions property, the axiom of choice allows us to assign $\Bs(\Bu)\in (\bR^n)^{\bZ_-}$ to each $\Bu\in ([-1,1]^d)^{\bZ_-}$, and hence define a functional $h:([-1,1]^d)^{\bZ_-} \to \bR$ by $h(\Bu) := \Ba \cdot\Bs(\Bu)_0 =y_0$. Thus, we can assign a causal time-invariant operator $H: ([-1,1]^d)^\bZ \to \bR^\bZ$ to the system such that $H^*=h$. When the echo state property holds, this operator is unique. The operator $H$ is continuous if and only if the mapping $\Bu \mapsto \Bs(\Bu)_0$ is a continuous function. In the next section, we study the universality of these operators.

\section{Universal approximation}

As mentioned in the introduction, the recent works of \cite{grigoryeva2018echo,gonon2021fading} showed that the echo state networks (\ref{general ESN}) are universal: Assume $\phi$ is a bounded Lipschitz continuous function. Let $F:([-1,1]^d)^\bZ\to \bR^{\bZ}$ be a continuous causal time-invariant operator, then for any $\epsilon>0$, for sufficiently large $n$, there exists an ESN (\ref{general ESN}) such that the corresponding causal time-invariant operator $F_{ESN}$ satisfies
\begin{equation}\label{app}
\sup_{\Bu\in ([-1,1]^d)^\bZ} \|F(\Bu) - F_{ESN}(\Bu)\|_\infty \le\epsilon.
\end{equation}
In this universal approximation theorem, the weights $W_1, W_2, \Ba, \Bb$ in the network (\ref{general ESN}) depend on the target operator $F$. However, in practice, the parameters $W_1, W_2, \Bb$ are drawn at random from certain given distribution and only the readout vector $\Ba$ are trained by linear regression using observed data related to the target operator $F$. Hence, this universal approximation theorem can not completely explain the empirical performance of echo state networks. 

In this section, our goal is to show that, with randomly generated weights, echo state networks are universal: For any $\epsilon>0$, for sufficiently large $n$, with high probability on $W_1, W_2, \Bb$, which are drawn from certain distribution, there exists $\Ba$ such that we can associate a causal time-invariant operator $F_{ESN}$ to the ESN (\ref{general ESN}) and the approximation bound (\ref{app}) holds. In the context of standard feed-forward neural networks, one can show that similar universal approximation theorem holds for random neural networks. This will be the building block of our main theorem of echo state networks.

\subsection{Universality of random neural networks}

It is well-known that feed-forward neural networks with one hidden layer are universal \cite{cybenko1989approximation,hornik1991approximation,leshno1993multilayer,pinkus1999approximation}. We recall the universal approximation theorem proved by \cite{leshno1993multilayer}, which has minimal requirements on the activation function.

\begin{theorem}\label{universality of NN}
If $\phi:\bR\to \bR$ is continuous and is not a polynomial, then for any compact set $\cX \subseteq \bR^d$, any function $f\in C(\cX)$ and $\epsilon>0$, there exists $a_i,b_i\in \bR$ and $\Bw_i\in \bR^d$, $i=1,\dots,n$, such that
\begin{equation}\label{app of NN}
\sup_{\Bx\in \cX} \left| f(\Bx) - \sum_{i=1}^n a_i \phi(\Bw_i \cdot \Bx+b_i) \right| \le \epsilon.
\end{equation}
\end{theorem}

In Theorem \ref{universality of NN}, the parameters $a_i,b_i,\Bw_i$ depend on the target function $f$. In order to take into account the fact that for ESN the inner weights $\Bw_i, b_i$ are randomly chosen and only $a_i$ are trained from data, we consider the random neural networks whose weights $(\Bw_i,b_i)$ are drawn from some probability distribution $\mu$. This motivates the following definition.

\begin{definition}\label{def universal}
Suppose $(\Bw_i^\intercal,b_i)_{i\in \bN}$ is a sequence of i.i.d. random vectors drawn from probability distribution $\mu$ defined on $\bR^{d+1}$. If for any compact set $\cX \subseteq \bR^d$, any function $f\in C(\cX)$ and $\epsilon,\delta \in (0,1)$, there exists $n\in\bN$ such that, with probability at least $1-\delta$, the inequality (\ref{app of NN}) holds for some $(a_i)_{1\le i\le n}$, then we say the pair $(\phi,\mu)$ is universal.
\end{definition}

\begin{remark}
We note that the readout weights $(a_i)_{1\le i\le n}$ depend on the randomly generated weights $(\Bw_i,b_i)_{1\le i\le n}$. In the definition, we only require the existence of neural networks that converge to the target function in probability. This requirement is actually equivalent to almost sure convergence. To see this, convergence in probability implies that there exists a sub-sequence $n_k$ such that there exists $f_{n_k}$, which is a linear combination of $\{\phi(\Bw_i \cdot \Bx+b_i)\}_{i=1}^{n_k}$, converging to the target function $f$ almost surely as $k\to \infty$. Notice that, for any $n\ge n_k$, $f_{n_k}$ is also a linear combination of $\{\phi(\Bw_i \cdot \Bx+b_i)\}_{i=1}^n$. Hence, there exists $f_n$, which is a linear combination of $\{\phi(\Bw_i \cdot \Bx+b_i)\}_{i=1}^n$, such that $f_n$ convergences to $f$ almost surely.
\end{remark}

The universality of random neural networks was widely studied in the context of extreme learning machine \cite{huang2006universal,huang2006extreme,huang2012extreme}. In particular, \cite{huang2006universal} used an incremental construction to establish the random universal approximation theorem in $L^2$-norm for bounded non-constant piecewise continuous activation function. The paper \cite{gonon2023approximation} studied the approximation of random neural networks in a Hilbert space setting and established the universal approximation of ReLU neural networks. The recent work of \cite{hart2020embedding} considered the approximation in $C^1$-norm and assumed that the activation function $\phi\in C^1(\bR)$ satisfying $0<\int_\bR |\phi'(x)|dx<\infty$. They argued that, since there exists a neural network $g$ that approximates the target function $f$ by universality of neural networks, there will eventually be some randomly generated samples $(\Bw_i^\intercal,b_i)$ that are close to the weights of $g$ and we can discard other samples by setting the corresponding $a_i=0$, hence the random universal approximation holds. It is possible to generalize their argument to the approximation of continuous functions. Nevertheless, we will give an alternative approach based on law of large numbers and show that $(\phi,\mu)$ is universal under very weak conditions. Our analysis will need the following uniform law of large numbers.

\begin{lemma}[\cite{jennrich1969asymptotic}, Theorem 2]\label{ULLN}
Let $\Bw_1,\dots,\Bw_n$ be i.i.d. samples from some probability distribution $\mu$ on $\Omega$. Suppose
\begin{enumerate}[label=\textnormal{(\arabic*)},parsep=0pt]
\item $\cX$ is compact;

\item $f(\Bx,\Bw)$ is continuous on $\Bx$ for each $\Bw\in \Omega$, and measurable on $\Bw$ for each $\Bx \in \cX$;

\item there exists $g(\Bw)$ with $|f(\Bx,\Bw) | \le g(\Bw)$ for all $\Bx \in \cX$ and $\bE_\Bw[g(\Bw)] <\infty$.
\end{enumerate}
Then $\bE_\Bw[f(\Bx,\Bw)]$ is continuous on $\Bx$ and 
\[
\sup_{\Bx \in \cX} \left| \frac{1}{n} \sum_{i=1}^n f(\Bx,\Bw_i) - \bE_\Bw[f(\Bx,\Bw)] \right| \to 0
\]
$\mu$-almost surely as $n\to\infty$.
\end{lemma}

Now, we give a sufficient condition for the pair $(\phi,\mu)$ to be universal. Our proof is a combination of the uniform law of large number (Lemma \ref{ULLN}) and the universality of neural networks (Theorem \ref{universality of NN}). 

\begin{theorem}\label{universal pair}
Suppose the continuous function $\phi:\bR\to \bR$ is not a polynomial and $|\phi(x)| \le C_1|x|^\alpha + C_2$ for some $\alpha,C_1,C_2 \ge 0$. If $\mu$ is a probability distribution with full support on $\bR^{d+1}$ such that $\int \|(\Bw^\intercal,b)\|_\infty^\alpha d\mu(\Bw^\intercal,b)<\infty$, then $(\phi,\mu)$ is universal.
\end{theorem}
\begin{proof}
By Hahn-Banach theorem, for any compact set $\cX\subseteq \bR^d$, the linear span of a function class $\cH\subseteq C(\cX)$ is dense in $C(\cX)$ if and only if
\[
\left\{ \nu \in M(\cX): \int_{\cX} h(\Bx) d\nu(\Bx) =0, \forall h\in \cH \right\} = \{ 0 \},
\]
where $M(\cX)$ is the dual space of $C(\cX)$, that is, the space of all signed Radon measures with finite total variation \cite[Chapter 7.3]{folland1999real}.

We consider the linear space 
\begin{multline}\label{H_mu}
\cH_\mu := \left\{ h(\Bx) = \int_{\bR^{d+1}} g(\Bw,b) \phi(\Bw \cdot \Bx+b) d\mu(\Bw^\intercal,b) \right. \\
\bigg.: g\in L^\infty(\mu) \bigg\}.
\end{multline}
Observe that $|g(\Bw,b) \phi(\Bw \cdot \Bx+b)| \le \|g\|_\infty (C_1|\Bw \cdot \Bx+b|^\alpha + C_2) \le C_{\cX} \|g\|_\infty \|(\Bw^\intercal,b)\|_\infty^\alpha + C_2\|g\|_\infty$, where $C_{\cX}$ is a constant depending on the compact set $\cX$. By assumption and Lemma \ref{ULLN}, any $h\in \cH_\mu$ is continuous and hence $\cH_\mu \subseteq C(\cX)$. Suppose $\nu \in M(\cX)$ satisfies $\int_{\cX} h(\Bx) d\nu(\Bx) =0$ for all $h\in \cH_\mu$. Then, by Fubini's theorem, 
\begin{align*}
0 =& \int_{\cX} \int_{\bR^{d+1}} g(\Bw,b) \phi(\Bw \cdot \Bx+b) d\mu(\Bw^\intercal,b) d\nu(\Bx) \\
=& \int_{\bR^{d+1}} g(\Bw,b) \int_{\cX} \phi(\Bw \cdot \Bx+b) d\nu(\Bx) d\mu(\Bw^\intercal,b),
\end{align*}
for all function $g\in L^\infty(\mu)$. Therefore,
\[
\int_{\cX} \phi(\Bw \cdot \Bx+b) d\nu(\Bx) =0, \quad  \mu\mbox{-almost surely}.
\]
Since this function is continuous and $\mu$ has full support, the equality holds for all $(\Bw^\intercal,b)\in \bR^{d+1}$. By Theorem \ref{universality of NN}, the linear span of $\{ \phi(\Bw \cdot \Bx+b): (\Bw^\intercal,b)\in \bR^{d+1} \}$ is dense in $C(\cX)$, which implies $\nu =0$. We conclude that $\cH_\mu$ is dense in $C(\cX)$. In other words, for any $f\in C(\cX)$ and $\epsilon>0$, there exists $h\in\cH_\mu$ such that $|f(\Bx) - h(\Bx)|\le \epsilon/2$ for all $\Bx\in \cX$.

By Lemma \ref{ULLN}, for any $\delta>0$, there exists $n\in \bN$ such that, with probability at least $1-\delta$ on the samples $(\Bw_i^\intercal,b_i)$ from $\mu$, 
\[
\sup_{\Bx \in \cX} \left| \frac{1}{n} \sum_{i=1}^n g(\Bw_i,b_i) \phi(\Bw_i \cdot \Bx+b_i) - h(\Bx) \right| \le \frac{\epsilon}{2}.
\]
By triangle inequality, we conclude that (\ref{app of NN}) holds with $a_i = \frac{1}{n} g(\Bw_i,b_i)$ and the pair $(\phi,\mu)$ is universal.
\end{proof}

In practice, it is more convenient to sample each weight independently from certain distribution $\mu$ on $\bR$, so that $(\Bw^\intercal,b)$ is sample from $\mu^{d+1}$, the $d+1$ products of $\mu$. The next corollary is a direct application of Lemma \ref{universal pair} to this situation.

\begin{corollary}\label{universal pair co}
Suppose the continuous function $\phi:\bR\to \bR$ is not a polynomial and $|\phi(x)| \le C_1|x|^\alpha + C_2$ for some $\alpha,C_1,C_2 \ge 0$. If $\mu$ is a probability distribution with full support on $\bR$ such that $\int |x|^\alpha d\mu(x)<\infty$, then for any $d\in\bN$, the pair $(\phi,\mu^{d+1})$ is universal.
\end{corollary}
\begin{proof}
When $0\le \alpha\le 1$,
\begin{align*}
&\int \| (\Bw^\intercal,b) \|_\infty^\alpha d\mu^{d+1}(\Bw^\intercal,b) \\
\le& \int  \|(\Bw^\intercal,b)\|_1^\alpha d\mu^{d+1}(\Bw^\intercal,b) \\
\le& \int \left(|b|^\alpha + \sum_{i=1}^d |w_i|^\alpha\right) d\mu^{d+1}(\Bw^\intercal,b) <\infty.
\end{align*}
When $\alpha>1$, since all norms of $\bR^{d+1}$ are equivalent,
\begin{align*}
&\int \| (\Bw^\intercal,b) \|_\infty^\alpha d\mu^{d+1}(\Bw^\intercal,b) \\
\le& C \int  \|(\Bw^\intercal,b)\|_\alpha^\alpha d\mu^{d+1}(\Bw^\intercal,b) \\
=& C \int \left(|b|^\alpha + \sum_{i=1}^d |w_i|^\alpha\right) d\mu^{d+1}(\Bw^\intercal,b) <\infty.
\end{align*}
In any cases, the pair $(\phi,\mu^{d+1})$ satisfies the condition in Lemma \ref{universal pair}, hence it is universal.
\end{proof}

So far, we have assumed that $\mu$ has full support. When the activation $\phi(x)=\sigma(x) = \max\{x,0\}$ is the ReLU function, this assumption can be weaken due to the absolute homogeneity of ReLU.

\begin{corollary}\label{universal pair relu}
For the ReLU function $\sigma$, if $\mu$ is a probability distribution on $\bR$ whose support contains the interval $[-r,r]$ for some $r>0$, then $(\sigma,\mu^{d+1})$ is universal for any $d\in\bN$.
\end{corollary}
\begin{proof}
We consider the continuous mapping $T:\bR^{d+1} \to [-r,r]^{d+1}$ defined by 
\[
T(\Bw^\intercal,b) = 
\begin{cases}
(\Bw^\intercal,b), \quad &\|(\Bw^\intercal,b)\|_\infty \le r\\
\frac{r}{\|(\Bw^\intercal,b)\|_\infty} (\Bw^\intercal,b) \quad &\|(\Bw^\intercal,b)\|_\infty > r.
\end{cases}
\]
Let $\gamma = T_\# \mu^{d+1}$ be the push-forward measure of $\mu^{d+1}$ under $T$ defined by $\gamma(S) = \mu^{d+1}(T^{-1}(S))$ for any measurable set $S\subseteq [-r,r]^{d+1}$. Then, the support of $\gamma$ is $[-r,r]^{d+1}$ by assumption. 

We firstly show that $(\sigma,\gamma)$ is universal. As in the proof of Theorem \ref{universal pair}, if $\nu \in M(\cX)$ satisfies $\int_{\cX} h(\Bx) d\nu(\Bx) =0$ for all $h\in \cH_\gamma$, then by Fubini's theorem, 
\[
\int_{\cX} \sigma(\tilde{\Bw} \cdot \Bx+ \tilde{b}) d\nu(\Bx) =0 
\]
holds for all $(\tilde{\Bw}^\intercal,\tilde{b})\in [-r,r]^{d+1}$. By the absolute homogeneity of $\sigma$, this equation actually holds for all $(\tilde{\Bw}^\intercal,\tilde{b})\in \bR^{d+1}$. The argument in the proof of Theorem \ref{universal pair} implies that $(\sigma,\gamma)$ is universal.

Observe that any sample $(\Bw_i^\intercal,b_i)$ from $\mu^{d+1}$ corresponding to a sample $(\tilde{\Bw}_i^\intercal,\tilde{b}_i) = T(\Bw_i^\intercal,b_i)$ from $\gamma = T_\#\mu^{d+1}$, and 
\[
\sum_{i=1}^n a_i \phi(\Bw_i \cdot \Bx+b_i) = \sum_{i=1}^n \frac{a_i}{c_i} \phi(\tilde{\Bw}_i \cdot \Bx+\tilde{b}_i),
\]
where $c_i =1$ if $\|(\Bw_i^\intercal,b_i)\|_\infty\le r$ and $c_i = \|(\Bw_i^\intercal,b_i)\|_\infty/r$ otherwise. We conclude that $(\sigma,\mu^{d+1})$ is universal.
\end{proof}

\subsection{Universality of echo state networks}

In this section, we will state and prove the random universal approximation theorem for echo state networks. Our analysis is based on the uniform law of large numbers and the universality of random feed-forward neural networks. For simplicity, we will assume that all internal weights in the network have the same distribution $\mu$ from now on. We make the following assumption on the activation function $\phi:\bR\to \bR$ and the distribution $\mu$, where we also assume that the input $\Bx\in [-2,2]^m$ (any compact set slightly larger than $[-1,1]^m$ is enough for our purpose).

\begin{assumption}\label{assum pair}
For any $m\in \bN$, the pair $(\phi,\mu^{m+1})$ is universal and, for any $\epsilon_0>0$, there exists a measurable mapping $\varphi_{m}:\bR^{m+1} \to \bR^m$ such that 
\begin{equation}\label{recover x}
\sup_{\Bx\in [-2,2]^m} \left\| \Bx - \bE_{(\Bw^\intercal,b) \sim \mu^{m+1}} [\varphi_{m}(\Bw,b)\phi(\Bw \cdot \Bx+b)] \right\|_\infty \le \epsilon_0
\end{equation}
and $\|\varphi_{m}(\Bw,b)\phi(\Bw \cdot \Bx+b)\|_\infty \le g_m(\Bw,b)$ for some $g_m$ satisfying $\bE[g_m(\Bw,b)]<\infty$.
\end{assumption}

Corollary \ref{universal pair co} gives sufficient condition for the pair $(\phi,\mu^{m+1})$ to be universal. Besides, in the proof of Theorem \ref{universal pair}, we have shown that the function class $\cH_\mu$ defined by (\ref{H_mu}) is dense in $C(\cX)$ for any compact set $\cX$. Therefore, if $\phi$ and $\mu$ satisfy Corollary \ref{universal pair co} (so the pair $(\phi,\mu^{m+1})$ satisfies Theorem \ref{universal pair}), then there exists a bounded mapping $\varphi_{m}$ that satisfies (\ref{recover x}) by the denseness of $\cH_\mu$. In summary, any activation $\phi$ and distribution $\mu$ that satisfy the Conditions in corollary \ref{universal pair co} also satisfy Assumption \ref{assum pair}. However, the function $\varphi_{m}$ in (\ref{recover x}) may be difficult to compute for general activation function. We will give explicit construction for the ReLU function in Corollary \ref{universality of relu ESN}.

By Lemma \ref{ULLN}, Assumption \ref{assum pair} ensures that if we have $n$ i.i.d. samples $(\Bw_1^\intercal,b_1),\dots,(\Bw_n^\intercal,b_n)$ from $\mu^{m+1}$, we can approximately reconstruct the input by
\[
\Bx \approx \frac{1}{n}\sum_{i=1}^n\varphi_{m}(\Bw_i,b_i)\phi(\Bw_i \cdot \Bx+b_i).
\]
In other words, the features $\phi(\Bw_i \cdot \Bx+b_i)$ contain enough information of the input and we can approximately recover it using $\varphi_{m}(\Bw_i,b_i)$ as coefficients. For echo state networks, this assumption guarantees that the hidden state does not lose too much information about the history of the input, and hence we can view it as a ``reservoir'' and approximate the desired function at any time steps. We make this idea precise in the next theorem.

\begin{theorem}\label{universality of ESN}
Suppose the activation $\phi:\bR\to\bR$ and the distribution $\mu$ satisfy Assumption \ref{assum pair}. For $n,m,d\in \bN$, let $(\Bw_1^\intercal,b_1),\dots,(\Bw_n^\intercal,b_n)$ be $n$ i.i.d. samples from $\mu^{md+1}$, and define the ESN
\begin{equation}\label{special ESN}
\begin{cases}
\Bs_t = \phi(W (\tfrac{1}{n}P \varphi_{md}(W,\Bb)\Bs_{t-1}+Q\Bu_t)+\Bb), \\
y_t = \Ba \cdot \Bs_t,
\end{cases}
\end{equation}
where $\Bu_t\in \bR^d$, $\Ba,\Bs_t\in\bR^n$, $W=(\Bw_1,\dots,\Bw_n)^\intercal\in \bR^{n\times md}$, $\Bb=(b_1,\dots,b_n)^\intercal \in \bR^n$ and we denote $\varphi_{md}(W,\Bb) :=(\varphi_{md}(\Bw_1,b_1),\dots,\varphi_{md}(\Bw_n,b_n))\in \bR^{md\times n}$ for measurable function $\varphi_{md} :\bR^{md+1} \to \bR^{md}$ and
\begin{equation}\label{matrix PQ}
P= 
\begin{pmatrix}
0_d & I_d & & \\
 & \ddots & \ddots & \\
& & 0_d & I_d \\
& & & 0_d
\end{pmatrix}
\in \bR^{md\times md}, Q=
\begin{pmatrix}
0_d \\
\vdots \\
0_d   \\
I_d
\end{pmatrix}
\in \bR^{md\times d},
\end{equation}
where $0_d$ and $I_d$ are the $d\times d$ zero matrix and identity matrix respectively. Then, for any continuous causal time-invariant operator $F:([-1,1]^d)^\bZ \to \bR^{\bZ}$, any $\epsilon,\delta\in (0,1)$, there exist $n,m\in\bN$ and measurable function $\varphi_{md}$ such that, with probability at least $1-\delta$, the system (\ref{special ESN}) has existence of solutions property and there exists $\Ba \in \bR^n$ so that the corresponding operator $F_{ESN}$ satisfies
\begin{equation}\label{acc}
\sup_{\Bu\in ([-1,1]^d)^\bZ} \|F(\Bu) - F_{ESN}(\Bu)\|_\infty \le\epsilon.
\end{equation}
\end{theorem}

\begin{proof}
The proof is divided into three steps.

\textbf{Step 1}: Choose parameters $(W,\Bb)$. 

We choose $m= m_F(\epsilon/3)$, then by the definition of $m_F(\epsilon/3)$ and $F^*_m$, for any $\Bu\in ([-1,1]^d)^\bZ$ and $t\in \bZ$, 
\[
|F(\Bu)_t- F^*_m(\Bu_{t-m+1:t})| \le \epsilon/3.
\]
Since $F^*_m:[-1,1]^{md} \to \bR$ is continuous by Proposition \ref{continuity}, we can continuously extend its domain to $[-2,2]^{md}$ by setting $F(x_1,\dots,x_{md}) = F(\tilde{x}_1,\dots,\tilde{x}_{md})$ where $\tilde{x}_i=\max\{\min\{x_i,1\},-1\}$.
By Assumption \ref{assum pair}, the pair $(\phi,\mu^{md+1})$ is universal. Hence, there exists $N_1\in \bN$ such that for any $n\ge N_1$, with probability at least $1-\delta/2$,
\begin{equation}\label{Fm app}
\sup_{\Bx\in [-2,2]^{md}} \left| F^*_m(\Bx) - \sum_{i=1}^n a_i \phi(\Bw_i \cdot \Bx+b_i) \right| \le \epsilon/3,
\end{equation}
for some $\Ba=(a_1,\dots,a_n)^\intercal \in \bR^n$.

By Assumption \ref{assum pair}, there exists $\varphi_{md}$ such that
\begin{multline*}
\sup_{\Bx\in [-2,2]^{md}} \left\| \Bx - \bE[\varphi_{md}(\Bw,b)\phi(\Bw \cdot \Bx+b)] \right\|_\infty \\
\le \frac{1}{2m} \omega_{F^*_m}^{-1}(\epsilon/3).
\end{multline*}
Lemma \ref{ULLN} implies that, 
\begin{multline*}
\sup_{\Bx\in [-2,2]^{md}}\Bigg\|\bE[\varphi_{md}(\Bw,b)\phi(\Bw \cdot \Bx+b)] - \Bigg. \\
\left. \frac{1}{n}\sum_{i=1}^n\varphi_{md}(\Bw_i,b_i)\phi(\Bw_i \cdot \Bx+b_i)\right\|_\infty \to 0,
\end{multline*}
almost surely as $n\to \infty$. Thus, there exists $N_2\in \bN$ such that for any $n\ge N_2$, with probability at least $1-\delta/2$,
\begin{multline*}
\sup_{\Bx\in [-2,2]^{md}}\left\|\Bx - \frac{1}{n}\sum_{i=1}^n\varphi_{md}(\Bw_i,b_i)\phi(\Bw_i \cdot \Bx+b_i)\right\|_\infty \\
\le \frac{1}{m} \omega_{F^*_m}^{-1}(\epsilon/3),
\end{multline*}
which can be written as, 
\begin{equation}\label{reconstruct app}
\sup_{\Bx\in [-2,2]^{md}}\left\|\Bx - \frac{1}{n} \varphi_{md}(W,\Bb) \phi(W \Bx+\Bb) \right\|_\infty \le \frac{1}{m} \omega_{F^*_m}^{-1}(\epsilon/3).
\end{equation}

Without loss of generality, we can assume that $\omega_{F^*_m}^{-1}(\epsilon/3)\le 1/2$ by setting $\epsilon$ sufficiently small. Now, we choose $n= \max\{N_1,N_2\}$, then with probability at least $1-\delta$, (\ref{Fm app}) and (\ref{reconstruct app}) hold. We will fix such parameters $(W,\Bb)$ that satisfy (\ref{Fm app}) and (\ref{reconstruct app}). 

\textbf{Step 2}: Establish the existence of solutions property. 

For any fixed $\Bu \in ([-1,1]^d)^{\bZ}$, we consider the operator $\psi:(\bR^{md})^\bZ \to (\bR^{md})^\bZ$ defined by
\[
\psi(\Br)_t := \tfrac{1}{n}P \varphi_{md}(W,\Bb)\phi(W \Br_{t-1}+\Bb)+Q\Bu_t \in \bR^{md}.
\]
For convenience, we will regard any vector in $\bR^{md}$ as $m$ vectors in $\bR^d$. To do this, for any $i=1,\dots,m$, we denote the $(i-1)d+1$ to $id$ entries of $\psi(\Br)_t$ by $\psi(\Br)_t[i]\in \bR^d$. Then, by the definition of matrices $P$ and $Q$, $\psi(\Br)_t[m] = \Bu_t$ and $\psi(\Br)_t[i] = (\tfrac{1}{n} \varphi_{md}(W,\Bb)\phi(W \Br_{t-1}+\Bb)) [i+1]$ for $1\le i\le m-1$. 

For $j=1,\dots,m$, we define the $j$ composition of $\psi$ inductively by $\psi^{(j)} = \psi \circ \psi^{(j-1)}$, where $\psi^{(0)}$ is the identity map. We are going to show that $\psi^{(m)}$ maps the compact set $([-\frac{3}{2},\frac{3}{2}]^{md})^\bZ$ into itself and then use fixed point theorem to establish the existence of solutions property. 
To do this, we assert that,  if $\Br\in ([-\tfrac{3}{2},\tfrac{3}{2}]^{md})^\bZ$, then, for $1\le i \le m-j$,
\[
\left\|\psi^{(j)}(\Br)_t[i] \right\|_\infty \le \tfrac{3}{2}+\tfrac{j}{m}\omega_{F^*_m}^{-1}(\epsilon/3),
\]
and for $m-j<k\le m$,
\[
\left\|\psi^{(j)}(\Br)_t[k] - \Bu_{t-m+k} \right\|_\infty \le \tfrac{m-k}{m}\omega_{F^*_m}^{-1}(\epsilon/3).
\]
We prove this by induction. When $j=1$, $\psi(\Br)_t[m] = \Bu_t \in [-1,1]^d$ and, by inequality (\ref{reconstruct app}), for $1\le i\le m-1$,
\begin{align*}
\left\|\psi(\Br)_t[i] \right\|_\infty \le& \left\|\Br_{t-1}[i+1] \right\|_\infty + \frac{1}{m} \omega_{F^*_m}^{-1}(\epsilon/3) \\
\le& \frac{3}{2}+ \frac{1}{m} \omega_{F^*_m}^{-1}(\epsilon/3).
\end{align*}
Hence, the assertion is true for $j=1$. 
Assume that the assertion is true for some $1\le j\le m-1$, we are going to show that it is true for $j+1$. Since $\psi^{(j)}(\Br)_t \in [-2,2]^{md}$, by inequality (\ref{reconstruct app}), for $1\le i \le m-1$,
\[
\left\|\psi^{(j+1)}(\Br)_t[i] - \psi^{(j)}(\Br)_{t-1}[i+1] \right\|_\infty \le \frac{1}{m} \omega_{F^*_m}^{-1}(\epsilon/3),
\]
and $\psi^{(j+1)}(\Br)_t[m] = \Bu_t$. Hence, by induction assumption, for $1\le i \le m-j-1$, 
\begin{align*}
\left\|\psi^{(j+1)}(\Br)_t[i] \right\|_\infty \le& \left\|\psi^{(j)}(\Br)_{t-1}[i+1]\right\|_\infty + \frac{1}{m} \omega_{F^*_m}^{-1}(\epsilon/3) \\
\le& \frac{3}{2}+ \frac{j+1}{m} \omega_{F^*_m}^{-1}(\epsilon/3),
\end{align*}
and for $m-j-1<k\le m-1$,
\begin{align*}
&\left\|\psi^{(j+1)}(\Br)_t[k] - \Bu_{t-m+k}\right\|_\infty \\
\le& \left\|\psi^{(j)}(\Br)_{t-1}[k+1] - \Bu_{t-1-m+k+1}\right\|_\infty + \frac{1}{m} \omega_{F^*_m}^{-1}(\epsilon/3) \\
\le& \frac{m-(k+1)}{m} \omega_{F^*_m}^{-1}(\epsilon/3) + \frac{1}{m} \omega_{F^*_m}^{-1}(\epsilon/3) \\
=& \frac{m-k}{m} \omega_{F^*_m}^{-1}(\epsilon/3).
\end{align*}
Hence, the assertion is true and we conclude that 
\[
\left\| \psi^{(m)} (\Br)_t - \Bu_{t-m+1:t} \right\|_\infty \le \omega_{F^*_m}^{-1}(\epsilon/3) \le \frac{1}{2}
\]
holds for $\Br\in ([-\tfrac{3}{2},\tfrac{3}{2}]^{md})^\bZ$, where we regard $\Bu_{t-m+1:t}\in ([-1,1]^d)^m$ as the vector $(\Bu_{t-m+1}^\intercal,\dots,\Bu_t^\intercal)^\intercal \in \bR^{md}$ to simplify the notation. This inequality implies that $\psi^{(m)}$ maps the compact convex set $([-\frac{3}{2},\frac{3}{2}]^{md})^\bZ$ into itself. By the continuity of $\psi^{(m)}$ and Schauder's Fixed Point Theorem (see \cite[Theorem 7.1]{shapiro2016fixed}), $\psi^{(m)}$ has at least a fixed point, that is, a point $\Br^* \in ([-\frac{3}{2},\frac{3}{2}]^{md})^\bZ$ such that $\psi^{(m)}(\Br^*) = \Br^*$. Note that the fixed point $\Br^*$ depends on $\Bu$ because the operator $\psi$ is defined through $\Bu$.

Now, we are ready to establish the existence of solution property for the system (\ref{special ESN}). For any $\Bu \in ([-1,1]^d)^{\bZ}$, we define $\widetilde{\Br} \in ([-2,2]^{md})^\bZ$ by $\widetilde{\Br}_t = \psi^{(j_t)}(\Br^*)_t$, where $j_t=t-\lfloor t/m \rfloor m \in \{0,1,\dots,m-1\}$. If $j_t=0$, then $\widetilde{\Br}_t = \Br^*_t$, $\widetilde{\Br}_{t-1} = \psi^{(m-1)}(\Br^*)_{t-1}$ and 
\begin{align*}
\psi(\widetilde{\Br})_t &= \tfrac{1}{n}P \varphi_{md}(W,\Bb)\phi(W \widetilde{\Br}_{t-1}+\Bb)+Q\Bu_t \\
&= \tfrac{1}{n}P \varphi_{md}(W,\Bb)\phi(W \psi^{(m-1)}(\Br^*)_{t-1}+\Bb)+Q\Bu_t \\
&= \psi(\psi^{(m-1)}(\Br^*))_t = \psi^{(m)}(\Br^*)_t \\
&= \Br^*_t = \widetilde{\Br}_t.
\end{align*}
If $j_t=j\in \{1,2,\dots,m-1\}$, then $\widetilde{\Br}_t = \psi^{(j)}(\Br^*)_t$, $\widetilde{\Br}_{t-1} = \psi^{(j-1)}(\Br^*)_{t-1}$ and
\begin{align*}
\psi(\widetilde{\Br})_t &= \tfrac{1}{n}P \varphi_{md}(W,\Bb)\phi(W \widetilde{\Br}_{t-1}+\Bb)+Q\Bu_t \\
&= \tfrac{1}{n}P \varphi_{md}(W,\Bb)\phi(W \psi^{(j-1)}(\Br^*)_{t-1}+\Bb)+Q\Bu_t \\
&= \psi(\psi^{(j-1)}(\Br^*))_t = \psi^{(j)}(\Br^*)_t \\
&= \widetilde{\Br}_t.
\end{align*}
In any cases, we have $\psi(\widetilde{\Br})_t = \widetilde{\Br}_t$, or equivalently,
\[
\widetilde{\Br}_t = \tfrac{1}{n}P \varphi_{md}(W,\Bb)\phi(W \widetilde{\Br}_{t-1}+\Bb)+Q\Bu_t.
\]
Since $\widetilde{\Br}$ is a fixed point of $\psi$, it is also a fixed point of $\psi^{(m)}$ and hence
\begin{align*}
\left\| \widetilde{\Br}_t - \Bu_{t-m+1:t}\right\|_\infty =& \left\| \psi^{(m)} (\widetilde{\Br})_t - \Bu_{t-m+1:t} \right\|_\infty \\
\le& \omega_{F^*_m}^{-1}(\epsilon/3),
\end{align*}
where we regard $\Bu_{t-m+1:t}$ as a vector of $\bR^{md}$ as before.

Let $\Bs_t := \phi(W \widetilde{\Br}_t+\Bb)$, then $\Bs$ satisfies the equation 
\[
\Bs_t = \phi(W (\tfrac{1}{n}P \varphi_{md}(W,\Bb)\Bs_{t-1}+Q\Bu_t)+\Bb).
\]
Hence, the system (\ref{special ESN}) has the existence of solution property with probability at least $1-\delta$. 

\textbf{Step 3}: Bound the approximation error.

For the chosen parameters $(W,\Bb)$, we let $\Ba = (a_1, \dots, a_n)^\intercal\in \bR^n$ be the vector that satisfies inequality (\ref{Fm app}), then
\[
\sup_{\Bx\in [-2,2]^{md}} \left| F^*_m(\Bx) - \Ba \cdot \phi(W \Bx+\Bb) \right| \le \epsilon/3.
\]
Since the system (\ref{special ESN}) has existence of solution property, we can associate a causal time-invariant operator $F_{ESN}$ to the system (\ref{special ESN}) such that
\[
F_{ESN}(\Bu)_t = \Ba \cdot \Bs_t = \Ba \cdot \phi(W \widetilde{\Br}_t+\Bb), \quad \Bu\in ([-1,1]^d)^\bZ.
\]
Therefore, for any $\Bu\in ([-1,1]^d)^\bZ$, we have $\| \widetilde{\Br}_t - \Bu_{t-m+1:t}\|_\infty \le \omega_{F^*_m}^{-1}(\epsilon/3)\le 1/2$ and 
\begin{align*}
&|F_{ESN}(\Bu)_t - F(\Bu)_t| = |\Ba \cdot \phi(W \widetilde{\Br}_t+\Bb) - F(\Bu)_t| \\
\le& |\Ba \cdot \phi(W \widetilde{\Br}_t+\Bb) - F^*_m(\widetilde{\Br}_t)| + |F^*_m(\widetilde{\Br}_t) - F^*_m(\Bu_{t-m+1:t})| \\
&\quad + |F^*_m(\Bu_{t-m+1:t}) - F(\Bu)_t| \\
\le& \epsilon/3 + \epsilon/3 +\epsilon/3 = \epsilon,
\end{align*}
which completes the proof.
\end{proof}

Theorem \ref{universality of ESN} establishes the random universal approximation for a special form of echo state networks. But we are  only able to prove the existence of solutions property for the constructed network. We note that the echo state property is established for certain ReLU neural networks constructed in \cite{gonon2023approximation,gonon2021fading}. It would be interesting to see whether this property also holds for the construction in Theorem \ref{universality of ESN} (it seems that some regularity assumption on the activation function is needed). In the following remark, we show that the networks constructed in the proof of Theorem \ref{universality of ESN} satisfy a property slightly weaker than the echo state property.

\begin{remark}\label{weak ESP}
Recall that the echo state property can be formulated in a forward version \cite[Theorem 2.1]{yildiz2012re}: there exits a sequence $\delta_{0:\infty}$ such that $\delta_t \to 0$ as $t\to \infty$ and, for any state sequences $\Bs_{0:\infty}$ and $\Bs'_{0:\infty}$ generated by the system (\ref{special ESN}) with the same input $\Bu_{1:\infty}$ and different initial states $\Bs_0$ and $\Bs_0'$ respectively, it holds $\|\Bs_t - \Bs'_t\|_\infty \le \delta_t$. Thus, the effect of the initial state vanishes as $t\to \infty$. In the proof of Theorem \ref{universality of ESN}, we have showed that $\Bs_t = \phi(W \widetilde{\Br}_t+\Bb)$ with $\| \widetilde{\Br}_t - \Bu_{t-m+1:t}\|_\infty \le \omega_{F^*_m}^{-1}(\epsilon/3)$, which can be smaller than any prescribed accuracy $\epsilon_0$ by setting $\epsilon$ small enough. Hence, for any fixed sample $(W,\Bb)$, the difference of the states $\|\Bs_t - \Bs'_t\|_\infty$ is bounded for large $t$, by the continuity of the activation $\phi$. For example, if $\phi$ is Lipschitz continuity with Lipschitz constant $L$, then $\|\Bs_t - \Bs'_t\|_\infty \le \| \phi(W \widetilde{\Br}_t+\Bb) - \phi(W \widetilde{\Br}'_t+\Bb) \|_\infty \le 2L \|W\| \epsilon_0$ by the triangle inequality, where $\|W\|$ denotes the operator norm induced by $\|\cdot\|_\infty$. Furthermore, the output sequences $y_t= \Ba \cdot \Bs_t$ and $y'_t= \Ba' \cdot \Bs'_t$ computed by the corresponding operators constructed in Theorem \ref{universality of ESN} satisfy the same approximation accuracy (\ref{acc}) for large $t$, which means the outputs are closed to each other. So, in our construction, the effect of the initial state is small.
\end{remark}

Let us take a closer look at the result of Theorem \ref{universality of ESN}. For a given pair of activation and distribution $(\phi,\mu)$, it is shown that if the internal weights are sampled as in (\ref{special ESN}), then the echo state network is a universal approximator with high probability. Notice that the sampling procedure in (\ref{special ESN}) depends on the function $\varphi_{md}$, which is implicitly given in Assumption \ref{assum pair}. As mentioned above, if the pair $(\phi,\mu)$ satisfies Corollary \ref{universal pair co}, we are guaranteed the existence of $\varphi_{md}$. However, it may be difficult to compute it explicitly. In the following, we show how to compute $\varphi_{md}$ for the widely used ReLU activation $\phi=\sigma$, when the distribution $\mu$ is symmetric.

\begin{proposition}\label{relu reconstruct input}
Let $\mu$ be a symmetric distribution on $\bR$ with variance $\int x^2 d\mu(x)=M_2<\infty$. Then, for the ReLU function $\sigma$ and any $m\in \bN$,
\[
\bE_{(\Bw^\intercal,b) \sim \mu^{m+1}} [\Bw \sigma(\Bw \cdot \Bx+b)] = \tfrac{M_2}{2} \Bx.
\]
\end{proposition}
\begin{proof}
Let us denote $\Bw = (w_1,\dots,w_m)^\intercal$ and $\Bx = (x_1,\dots,x_m)^\intercal$. We need to show that $2\bE[w_j\sigma(\Bw \cdot \Bx+b)] = x_j M_2$ for any $j=1,\dots,m$. By symmetry, 
\[
\bE [w_j\sigma(\Bw \cdot \Bx+b)]  = -\bE [w_j\sigma(-\Bw \cdot \Bx - b)].
\]
Using the equality $\sigma(x) - \sigma(-x) = x$ for the ReLU function, we have
\begin{align*}
&2 \bE [w_j \sigma(\Bw \cdot \Bx+b)] \\
=& \bE [w_j \sigma(\Bw \cdot \Bx+b)] - \bE [w_j \sigma(-\Bw \cdot \Bx-b)] \\
=& \bE [w_j (w_1 x_1+ \dots + w_m x_m +b)] \\
=& x_j \bE [w_j^2] = x_j M_2,
\end{align*}
where we use the fact that $\bE w_jw_ix_i = \bE w_j \bE w_ix_i = 0$ for $i\neq j$, because the random variables $w_i$ and $w_j$ are independent and symmetric.
\end{proof}

By Proposition \ref{relu reconstruct input}, for the ReLU function $\sigma$ and symmetric distribution $\mu$, we can choose $\varphi_{m}(\Bw,b) = \frac{2}{M_2} \Bw$ in Assumption \ref{assum pair} with error $\epsilon_0=0$. As a corollary, we can obtain the random universal approximation theorem for ReLU echo state networks under very weak assumption on the distribution $\mu$.

\begin{corollary}\label{universality of relu ESN}
Let $\mu$ be a symmetric probability distribution on $\bR$ whose support contains the interval $[-r,r]$ for some $r>0$, and variance $\int x^2 d\mu(x)=M_2<\infty$. For $n,m,d\in\bN$, let $(\Bw_1^\intercal,b_1),\dots,(\Bw_n^\intercal,b_n)$ be $n$ i.i.d. samples from $\mu^{md+1}$, and define the ReLU ESN
\begin{equation}\label{relu ESN}
\begin{cases}
\Bs_t = \sigma(W (\tfrac{2}{n M_2}PW^\intercal \Bs_{t-1}+Q\Bu_t)+\Bb), \\
y_t = \Ba \cdot \Bs_t,
\end{cases}
\end{equation}
where $\Bu_t\in \bR^d$, $\Ba,\Bs_t\in\bR^n$, $W=(\Bw_1,\dots,\Bw_n)^\intercal\in \bR^{n\times md}$, $\Bb=(b_1,\dots,b_n)^\intercal \in \bR^n$ and $P,Q$ are defined by (\ref{matrix PQ}). Then, for any continuous causal time-invariant operator $F:([-1,1]^d)^\bZ \to \bR^{\bZ}$, any $\epsilon,\delta\in (0,1)$, there exist $n,m\in\bN$ such that, with probability at least $1-\delta$, the system (\ref{relu ESN}) has existence of solutions property and there exists $\Ba \in \bR^n$ so that the corresponding operator $F_{ESN}$ satisfies
\[
\sup_{\Bu\in ([-1,1]^d)^\bZ} \|F(\Bu) - F_{ESN}(\Bu)\|_\infty \le\epsilon.
\]
\end{corollary}
\begin{proof}
For any $m\in \bN$, the pair $(\sigma,\mu^{m+1})$ is universal by Corollary \ref{universal pair relu} and satisfies Assumption \ref{assum pair} with $\varphi_{m}(\Bw,b) = \frac{2}{M_2} \Bw$ by Proposition \ref{relu reconstruct input}. The result then follows from Theorem \ref{universality of ESN} by observing that $\varphi_{md}(W,\Bb) = (\frac{2}{M_2} \Bw_1,\dots, \frac{2}{M_2} \Bw_n) = \frac{2}{M_2} W^\intercal$.
\end{proof}

\section{Approximation bounds for ReLU networks}

We have shown that the ReLU echo state network (\ref{relu ESN}) with randomly generated weights can approximate any continuous causal time-invariant operator with high probability. In this section, we derive approximation bounds for sufficiently regular functions. To simplify the analysis, we will assume that each random weight in (\ref{relu ESN}) is sampled from the uniform distribution on $[-1/2,1/2]$. Similar results can be derived for distributions $\mu$ with bounded support. To do that, one can modify Assumption \ref{regularity assumption} by replacing the expectation (\ref{integral representation}) over uniform distribution by the expectation over $\mu$. Then, the proofs of Lemma \ref{app integral} and Theorem \ref{app bound for ReLU} can be applied.

The key ingredients of our approximation results are certain
integral representation of sufficiently regular functions and the random neural networks approximation of these functions. We make the following regularity assumption on the target continuous causal time-invariant operator $F$.

\begin{assumption}\label{regularity assumption}
The function $F^*_m:[-1,1]^{md} \to \bR$ defined by (\ref{F^*_m}) can be represented as 
\begin{equation}\label{integral representation}
F^*_m(\Bx) = \int_{-\frac{1}{2}}^{\frac{1}{2}}\int_{[-\frac{1}{2},\frac{1}{2}]^{md}} g_m(\Bw,b)\sigma(\Bw \cdot \Bx+b) d\Bw db,
\end{equation}
for some function $g_m$ with uniform bound $|g_m(\Bw,b)|\le B_m$.
\end{assumption}

Neural network approximation of functions with integral representations was initially studied by the classical work of Barron \cite{barron1993universal}. This result has been generalized and widely discussed in recent articles \cite{klusowski2018approximation,weinan2019priori,weinan2022barron,weinan2022representation,gonon2023approximation}.
In particular, \cite{klusowski2018approximation,gonon2023approximation} showed that sufficient regular functions (whose Fourier transforms decay sufficiently fast) can be represented by integral forms similar to (\ref{integral representation}). The following proposition is a modification of these results. The proof is deferred to the appendix.

\begin{proposition}\label{Fourier assumption}
Suppose $f:\bR^d \to \bR$ admits a Fourier representation
\[
f(\Bx) = \int_{\bR^d} \hat{f}(\Bw) e^{i \Bw \cdot \Bx} d \Bw, \quad \Bx \in [-1,1]^d,
\]
and $\sup_{r>0} \sup_{\|\Bw\|_1=1/2} \max(1,r^{2d+4}) |\hat{f}(r\Bw)| \le B$. Then there exists $g:[-1/2,1/2]^{d+1} \to \bR$ with $|g(\Bw,b)|\le 40 B$ such that
\[
f(\Bx) = \int_{-\frac{1}{2}}^{\frac{1}{2}} \int_{[-\frac{1}{2},\frac{1}{2}]^d} g(\Bw,b) \sigma(\Bw \cdot \Bx+b) d\Bw db, \ \Bx \in [-1,1]^d.
\]
\end{proposition}

The next lemma gives error bound for the random approximation of the integral form (\ref{integral representation}). This result can be seen as a quantitative version of Lemma \ref{ULLN}. Intuitively, the integral can be regarded as an expectation over uniform distribution. Hence, the law of large numbers implies that the empirical average converges to the expectation. We can bound this approximation error by using statistical learning theory \cite{shalevshwartz2014understanding} and concentration inequalities \cite{boucheron2013concentration}. Specifically, we view the uniform approximation error as the supremum of difference between the expectation $f(\Bx)$ and its empirical average. Then, we can apply concentration inequality to bound the deviation of this supremum. The expectation of the supremum is controlled by the Rademacher complexity \cite{bartlett2002rademacher}. The proof is given in the appendix.

\begin{lemma}\label{app integral}
Let
\[
f(\Bx) = \int_{[-\frac{1}{2},\frac{1}{2}]^d} g(\Bw)\sigma(\Bw \cdot \Bx) d\Bw
\]
with $|g(\Bw)| \le B$. Then, with probability at least $1-\delta$, 
\begin{multline*}
\sup_{\Bx\in [-r,r]^d} \left| f(\Bx) - \frac{1}{n} \sum_{i=1}^n g(\Bw_i)\sigma(\Bw_i \cdot \Bx) \right| \\
\le Brd \left( 2\sqrt{\frac{2d\log(n+1)}{n}} + \sqrt{\frac{\log(2/\delta)}{2n}} \right),
\end{multline*}
where $\Bw_{1:n} := (\Bw_i)_{1\le i \le n}$ are i.i.d. samples from the uniform distribution on $[-1/2,1/2]^d$.
\end{lemma}

Now, we are ready to prove our approximation bound for ReLU echo state networks. The proof is similar to the proof of Theorem \ref{universality of ESN}, in which we reduce the approximation of operator $F$ to the approximation of function $F_m^*$ and construct an ESN to approximate $F_m^*$. Here, we quantify this approximation error by Lemma \ref{app integral}.

\begin{theorem}\label{app bound for ReLU}
For $n,m,d\in\bN$, let $(\Bw_1^\intercal,b_1),\dots,(\Bw_n^\intercal,b_n)$ be $n$ i.i.d. samples from the uniform distribution on $[-1/2,1/2]^{md+1}$, and define the ReLU ESN
\[
\begin{cases}
\Bs_t = \sigma(W (\tfrac{24}{n}PW^\intercal \Bs_{t-1}+Q\Bu_t)+\Bb), \\
y_t = \Ba \cdot \Bs_t,
\end{cases}
\]
where $\Bu_t\in \bR^d$, $\Ba,\Bs_t\in\bR^n$, $W=(\Bw_1,\dots,\Bw_n)^\intercal\in \bR^{n\times md}$, $\Bb=(b_1,\dots,b_n)^\intercal \in \bR^n$ and $P,Q$ are defined by (\ref{matrix PQ}). Then, for sufficiently large $n/\log(n+1)\ge c(m,\delta,d)$ and any continuous causal time-invariant operator $F:([-1,1]^d)^\bZ \to \bR^{\bZ}$ satisfying Assumption \ref{regularity assumption}, with probability at least $1-\delta$, the system has existence of solutions property and there exists $\Ba \in \bR^n$ so that the corresponding operator $F_{ESN}$ satisfies
\begin{multline*}
\sup_{\Bu\in ([-1,1]^d)^\bZ} \|F(\Bu) - F_{ESN}(\Bu)\|_\infty \\
\le C(m,\delta,d) \sqrt{\frac{\log(n+1)}{n}} + \cE_F(m),
\end{multline*}
where we can choose $C(m,\delta,d)= \sqrt{2}B_m(md+1)(3m^2d+1) ( 4\sqrt{md+1} + \sqrt{\log(4md/\delta)} )$ and $c(m,\delta,d) = 1152m^2(md+1)^2(4\sqrt{md+1}+\sqrt{\log(4md/\delta)})^2$.
\end{theorem}

\begin{proof}
For any $t\in \bZ$, by equation (\ref{error of F^*_m}),
\[
\sup_{\Bu \in ([-1,1]^d)^\bZ} |F(\Bu)_t - F_m^*(\Bu_{t-m+1:t})| \le \cE_F(m).
\]
By Assumption \ref{regularity assumption}, we can extend the domain of $F_m^*$ to $[-2,2]^{md}$ by the integral representation (\ref{integral representation}). Using Lemma \ref{app integral}, it holds with probability at least $1-\delta/2$ that
\begin{equation}\label{F_m* random app}
\begin{aligned}
&\sup_{\Bx\in [-2,2]^{md}} \left| F_m^*(\Bx) - \frac{1}{n} \sum_{i=1}^n g(\Bw_i,b_i)\sigma(\Bw_i \cdot \Bx+b_i) \right| \\
\le& \sqrt{2}B_m(md+1) \left( 4\sqrt{\frac{(md+1)\log(n+1)}{n}} + \sqrt{\frac{\log(4/\delta)}{n}} \right) \\
=&:\cE_1.
\end{aligned}
\end{equation}
On the other hand, Proposition \ref{relu reconstruct input} shows that $\Bx = 24\bE \Bw \sigma(\Bw \cdot \Bx+b)$. Hence, applying Lemma \ref{app integral} to each coordinate of $\Bx$, it holds with probability at least $1-\delta/2$ that
\begin{equation}\label{x random app}
\begin{aligned}
&\sup_{\Bx\in [-2,2]^{md}} \left\| \Bx - \frac{24}{n} \sum_{i=1}^n \Bw_i \sigma(\Bw_i \cdot \Bx+b_i) \right\|_\infty \\
\le& 12\sqrt{2}(md+1) \\
&\quad \times \left( 4\sqrt{\frac{(md+1)\log(n+1)}{n}} + \sqrt{\frac{\log(4md/\delta)}{n}} \right) \\
=&:\cE_2 \le \frac{1}{2m},
\end{aligned}
\end{equation}
where the last inequality holds if $n/\log(n+1)\ge 1152m^2(md+1)^2(4\sqrt{md+1}+\sqrt{\log(4md/\delta)})^2$. We note that, in matrix form,
\[
\Bx - \frac{24}{n} \sum_{i=1}^n \Bw_i \sigma(\Bw_i \cdot \Bx+b_i) = \Bx - \frac{24}{n} W^\intercal \sigma(W \Bx +\Bb).
\]

Therefore, with probability at least $1-\delta$ over parameters $(W,\Bb)$, inequalities (\ref{F_m* random app}) and (\ref{x random app}) hold. For such parameters $(W,\Bb)$, Step 2 of the proof of Theorem \ref{universality of ESN} shows that, for any $\Bu\in ([-1,1]^d)^\bZ$, there exists $\widetilde{\Br}\in ([-2,2]^{md})^\bZ$ such that
\[
\widetilde{\Br}_t = \tfrac{24}{n}P W^\intercal\sigma(W \widetilde{\Br}_{t-1}+\Bb)+Q\Bu_t,
\]
and
\[
\left\| \widetilde{\Br}_t - \Bu_{t-m+1:t}\right\|_\infty  \le m\cE_2,
\]
where we regard $\Bu_{t-m+1:t}$ as a vector of $\bR^{md}$ in the natural way. By the integral representation (\ref{integral representation}), we have
\begin{multline*}
|F^*_m(\widetilde{\Br}_t) - F^*_m(\Bu_{t-m+1:t})| \\
\le B_m \int_{[-\frac{1}{2},\frac{1}{2}]^{md}} \|\Bw\|_1 \left\| \widetilde{\Br}_t - \Bu_{t-m+1:t}\right\|_\infty d\Bw \le \frac{B_m m^2 d}{4} \cE_2.
\end{multline*}
Let $\Bs_t := \sigma(W \widetilde{\Br}_t+\Bb)$, then $\Bs$ satisfies the equation 
\[
\Bs_t = \sigma(W (\tfrac{24}{n}PW^\intercal \Bs_{t-1}+Q\Bu_t)+\Bb).
\]
Hence, the system has the existence of solution property with probability at least $1-\delta$ and we can we can associate a causal time invariant operator $F_{ESN}$ to the system such that
\[
F_{ESN}(\Bu)_t = \Ba \cdot \Bs_t = \Ba \cdot \sigma(W \widetilde{\Br}_t+\Bb), \quad \Bu\in ([-1,1]^d)^\bZ.
\]
Let $\Ba = \frac{1}{n}(g(\Bw_1,b_1),\dots,g(\Bw_n,b_n))^\intercal$, then by inequality (\ref{F_m* random app}),
\begin{align*}
&|F_{ESN}(\Bu)_t - F(\Bu)_t| = |\Ba \cdot \sigma(W \widetilde{\Br}_t+\Bb) - F(\Bu)_t| \\
\le& |\Ba \cdot \sigma(W \widetilde{\Br}_t+\Bb) - F^*_m(\widetilde{\Br}_t)| + |F^*_m(\widetilde{\Br}_t) - F^*_m(\Bu_{t-m+1:t})| \\
&\qquad + |F^*_m(\Bu_{t-m+1:t}) - F(\Bu)_t| \\
\le& \cE_1 + \frac{B_m m^2 d}{4} \cE_2 + \cE_F(m) \\
\le& \sqrt{2}B_m(md+1) \left( (12m^2d+4)\sqrt{\frac{(md+1)\log(n+1)}{n}} \right. \\
&\qquad \left. + \sqrt{\frac{\log(4/\delta)}{n}} + 3m^2d \sqrt{\frac{\log(4md/\delta)}{n}} \right)  + \cE_F(m) \\
\le& C(m,\delta,d) \sqrt{\frac{\log(n+1)}{n}} + \cE_F(m),
\end{align*}
where $C(m,\delta,d)= \sqrt{2}B_m(md+1)(3m^2d+1) ( 4\sqrt{md+1} + \sqrt{\log(4md/\delta)} )$.
\end{proof}

\begin{remark}
Similar approximation bound has been obtained in \cite{gonon2023approximation}. They made assumptions that the input sequences are sampled from some probability distribution and the target operator is Lipschitz continuous, and proved a bound for the expected $L^2$ approximation error. In contrast, we generalize the Lipschitz continuity to general modulus of continuity, and obtain a high probability bound for the uniform approximation error in Theorem \ref{app bound for ReLU}. Note that, if $B_m$ is at most polynomial on $m$ and $\cE_F(m)$ decays exponentially, say $\cE_F(m)\le c_0 e^{-c_1 m}$ for some $c_0,c_1>0$, then we can choose $m \approx (2c_1)^{-1} \log n$ and the approximation error is $\widetilde{\cO}(\sqrt{\log(1/\delta)/n})$, where the notation $\widetilde{\cO}$ hides poly-logarithmic factors of $n$. In this case, the approximation rate is the same as \cite[Theorem 2]{gonon2023approximation}.
\end{remark}

\section{Conclusions}

We have analyzed the approximation capacity of echo state networks with randomly generated internal weights. Under weak assumption on the activation function, we constructed a sampling procedure for the internal weights so that the echo state networks can approximate any continuous casual time-invariant operators with prescribed accuracy. For ReLU activation, we explicitly computed the random weights in our construction and derived the approximation bounds for sufficiently regular operators.

Our results also raise some questions for future study. For example, our construction of the sampling procedure relies on the function $\varphi_{m}$ in Assumption \ref{assum pair}. In other words, we only establish the universality of echo state networks when the parameters are sampled in a special way. In practical applications, the parameters in echo state networks are directly sampled from some simple distributions, and they do not have the particular form constructed in this paper. It would be interesting to extend the universality result to these ``practical'' sampling procedures. Another interesting question is the echo state property of the constructed networks. In Theorem \ref{universality of ESN} and Remark \ref{weak ESP}, we only prove that the constructed networks have existence of solution property and satisfy a property weaker than the the echo state property. In order to establish the echo state property, we think more regularity condition on the activation function is needed.

\appendix

\section*{Proof of Proposition \ref{continuity}}

For the first part, by the definition of $m=m_F(\epsilon/3)$,
\[
\sup_{\Bu\in ([-1,1]^d)^{\bZ_-}} |F^*(\Bu) -  F^*_m(\Bu_{-m+1:0})| \le \epsilon/3.
\]
Since the weighting sequence is decreasing, 
\[
\|\Bu_{-m+1:0} - \Bv_{-m+1:0}\|_\infty \le \eta_{m-1}^{-1} \|\Bu-\Bv\|_\eta \le \omega_{F^*_m}^{-1}(\epsilon/3),
\]
which implies $|F^*_m(\Bu_{-m+1:0}) - F^*_m(\Bv_{-m+1:0})| \le \epsilon/3$. Therefore,
\begin{align*}
&|F^*(\Bu) - F^*(\Bv)| \\
\le& |F^*(\Bu) - F^*_m(\Bu_{-m+1:0})| + |F^*_m(\Bu_{-m+1:0}) \\
&\qquad - F^*_m(\Bv_{-m+1:0})| + |F^*_m(\Bv_{-m+1:0}) - F^*(\Bv)| \\
\le& \epsilon/3 + \epsilon/3 + \epsilon/3 = \epsilon.
\end{align*}

For the second part, we observe that, for any $m\in \bN$ and $\Bu\in ([-1,1]^d)^{\bZ_-}$,
\[
\| \Bu - (\dots,\boldsymbol{0},\Bu_{-m+1:0})\|_\eta = \sup_{t\le -m} \eta_{-t}\|\Bu_t\|_\infty \le \eta_m.
\]
Then, by the definition of $\omega_{F^*}(\delta;\eta)$,
\begin{align*}
\cE_F(m) &=  \sup_{\Bu\in ([-1,1]^d)^{\bZ_-}} |F^*(\Bu) - F^*(\dots,\boldsymbol{0},\Bu_{-m+1:0})| \\
&\le \omega_{F^*}(\eta_m,\eta).
\end{align*}
If $m$ satisfies $\eta_m \le \omega_{F^*}^{-1}(\epsilon;\eta)$, then $\cE_F(m) \le \omega_{F^*}(\eta_m,\eta) \le \epsilon$. Hence, 
\begin{align*}
m_F(\epsilon) &= \min \{m\in \bN: \cE_F(m) \le \epsilon \} \\
&\le \min\left\{ m\in \bN: \eta_m \le \omega_{F^*}^{-1}(\epsilon;\eta) \right\}.
\end{align*}
Finally, by the definition of $F^*_m$ and $\|(\cdots, \boldsymbol{0},\Bu_{-m+1:0})-(\cdots, \boldsymbol{0},\Bu_{-m+1:0}')\|_\eta \le \|\Bu_{-m+1:0}-\Bu_{-m+1:0}'\|_\infty$, we have
\begin{align*}
&\omega_{F^*_m}(\delta) \\
=& \sup\{ |F^*(\cdots, \boldsymbol{0},\Bu_{-m+1:0}) - F^*(\cdots, \boldsymbol{0},\Bu_{-m+1:0}')| \\
&\qquad\qquad : \|\Bu_{-m+1:0}-\Bu_{-m+1:0}'\|_\infty \le \delta \} \\ 
\le& \omega_{F^*}(\delta;\eta). \qedhere
\end{align*}

\section*{Proof of Proposition \ref{Fourier assumption}} \label{Appendix A}

The proof is a modification of the argument in \cite{klusowski2018approximation,gonon2023approximation}. Firstly, observe that, for any $t\in \bR$,
\[
-\int_0^{\infty} \sigma(t-b) e^{ib} + \sigma(-t-b)e^{-ib} db = e^{it} -it-1.
\]
For any $\Bx \in [-1,1]^d$, letting $t=\Bw\cdot \Bx$, we have
\begin{align*}
&f(\Bx) - \nabla f(0) \cdot \Bx - f(0) \\
= & \int_{\bR^d} (e^{i\Bw \cdot \Bx} -i\Bw \cdot \Bx-1) \hat{f}(\Bw) d\Bw \\
=& -\int_{\bR^d} \int_0^\infty (\sigma(\Bw\cdot \Bx-b) e^{ib} + \sigma(-\Bw\cdot \Bx-b)e^{-ib}) \\
&\qquad \times \hat{f}(\Bw) dbd\Bw \\
=& - \int_{\bR^d} \int_{-\infty}^0 \sigma(\Bw\cdot \Bx+b) e^{-ib} \hat{f}(\Bw) db d\Bw \\
&\qquad - \int_{\bR^d} \int_{-\infty}^0 \sigma(\Bw\cdot \Bx+b) e^{ib} \hat{f}(-\Bw) db d\Bw \\
=& \int_{\bR^d} \int_{-\|\Bw\|_1}^0 \Re [-e^{-ib} \hat{f}(\Bw) -e^{ib} \hat{f}(-\Bw)] \\
&\qquad \times \sigma(\Bw\cdot \Bx+b) db d\Bw,
\end{align*}
where, in the last equality, we use the fact that $\sigma(\Bw\cdot \Bx+b)=0$ for $b< -\|\Bw\|_1$. 
Since $f(0), \nabla f(0) \cdot \Bx \in\bR$, we have $\int_{\bR^d} \Im [\hat{f}(\Bw)] d\Bw =0$ and $\int_{\bR^d} (\Bw \cdot \Bx) \Re[\hat{f}(\Bw)] d\Bw =0$, which implies
\begin{align*}
&\nabla f(0) \cdot \Bx + f(0) \\
=& \int_{\bR^d} -(\Bw \cdot \Bx) \Im[\hat{f}(\Bw)] + \Re[\hat{f}(\Bw)] d\Bw \\
=& \int_{\bR^d} \int_{0}^{1/2} (\Bw \cdot \Bx + b) (8\Re[\hat{f}(\Bw)] - \Im[\hat{f}(\Bw)]) db d\Bw \\
=& \int_{\bR^d} \int_{0}^{1/2} (\sigma(\Bw \cdot \Bx + b) - \sigma(-\Bw \cdot \Bx - b) ) \\
&\qquad \times \Re[(8+i)\hat{f}(\Bw)] db d\Bw \\
=& \int_{\bR^d} \int_{0}^{1/2} \Re[(8+i)\hat{f}(\Bw)] \sigma(\Bw \cdot \Bx + b) db d\Bw \\
&\qquad- \int_{\bR^d} \int_{-1/2}^{0} \Re[(8+i)\hat{f}(-\Bw)] \sigma(\Bw \cdot \Bx + b) db d\Bw.
\end{align*}
By Fubini's theorem, we conclude that
\[
f(\Bx) = \int_{\bR^{d+1}} h(\Bw,b) \sigma(\Bw \cdot \Bx + b) d\Bw db,
\]
with
\begin{align*}
h(\Bw,b) :=& 1_{[-\|\Bw\|_1,0]}(b) \Re [-e^{-ib} \hat{f}(\Bw) -e^{ib} \hat{f}(-\Bw)] \\
& \quad+ 1_{[0,1/2]}(b) \Re[(8+i)\hat{f}(\Bw)] \\
&\quad - 1_{[-1/2,0]}(b) \Re[(8+i)\hat{f}(-\Bw)],
\end{align*}
where $1_S$ denote the indicator function of the set $S$. Note that, if $|b|>\|\Bw\|_1$ and $|b|>1/2$, then $h(\Bw,b)=0$.

Next, we denote $\cA=\{\Bw\in \bR^d: \|\Bw\|_1>1/2\}$, $\cB = \{\Bw\in \bR^d: 0<\|\Bw\|_1< 1/2\}$ and  consider the mapping
\[
T: \cB \to \cA, \quad T(\Bw) = \frac{\Bw}{4\|\Bw\|_1^2}.
\]
Then, the Jacobian of $T$ is
\[
\nabla T(\Bw) = \frac{1}{4\|\Bw\|_1^2} \left( I_d - 2 \frac{\Bw}{\|\Bw\|_1} \sgn(\Bw)^\intercal \right)
\]
with determinant
\begin{align*}
|\det(\nabla T(\Bw))| =& \frac{1}{(2\|\Bw\|_1)^{2d}} \left| 1- 2 \sgn(\Bw)^\intercal \frac{\Bw}{\|\Bw\|_1} \right| \\
=& \frac{1}{(2\|\Bw\|_1)^{2d}},
\end{align*}
where we use Sylvester's determinant theorem. The change of variables formula implies
\begin{align*}
&\int_\bR\int_{\cA} h(\Bw,b) \sigma(\Bw \cdot \Bx + b) d\Bw db \\
=& \int_\bR\int_{\cB} h \left(\frac{\Bw}{4\|\Bw\|_1^2},b \right)\sigma\left(\frac{\Bw\cdot \Bx}{4\|\Bw\|_1^2} + b\right) \frac{1}{(2\|\Bw\|_1)^{2d}} d\Bw db \\
=& \int_\bR\int_{\cB} h \left(\frac{\Bw}{4\|\Bw\|_1^2},\frac{b}{4\|\Bw\|_1^2} \right)\sigma(\Bw\cdot \Bx + b) \\
&\qquad \times \frac{1}{(2\|\Bw\|_1)^{2d+4}} d\Bw db,
\end{align*}
where we use the homogeneity of ReLU in the last equality. Hence, 
\begin{align*}
&f(\Bx) \\
=& \int_\bR \int_{\cB} h(\Bw,b) \sigma(\Bw \cdot \Bx + b) d\Bw db \\
&\qquad + \int_\bR \int_{\cA} h(\Bw,b) \sigma(\Bw \cdot \Bx + b) d\Bw db \\
=& \int_\bR \int_{\cB} \left( h(\Bw,b) + \frac{1}{(2\|\Bw\|_1)^{2d+4}} h \left(\frac{\Bw}{4\|\Bw\|_1^2},\frac{b}{4\|\Bw\|_1^2} \right) \right) \\
&\qquad \times \sigma(\Bw\cdot \Bx + b) d\Bw db.
\end{align*}
We define
\begin{multline*}
g(\Bw,b)\\:= 1_{\cB}(\Bw) \left( h(\Bw,b) + \frac{1}{(2\|\Bw\|_1)^{2d+4}} h \left(\frac{\Bw}{4\|\Bw\|_1^2},\frac{b}{4\|\Bw\|_1^2} \right) \right).
\end{multline*}
Then, $g(\Bw,b)=0$ for $\|\Bw\|_1>1/2$. If $\|\Bw\|_1\le 1/2$ and $|b|>1/2$, then $|b|>\|\Bw\|_1$ and $|b|/(4\|\Bw\|_1^2) >1/2$, and hence $h(\Bw,b) = h (\frac{\Bw}{4\|\Bw\|_1^2},\frac{b}{4\|\Bw\|_1^2} ) =0$, which implies $g(\Bw,b)=0$. This shows that
\[
f(\Bx) = \int_{-\frac{1}{2}}^{\frac{1}{2}} \int_{[-\frac{1}{2},\frac{1}{2}]^d} g(\Bw,b) \sigma(\Bw \cdot \Bx+b) d\Bw db, \ \Bx \in [-1,1]^d.
\]
Finally, for $0<\|\Bw\|_1\le 1/2$, by the definition of $h(\Bw,b)$ and the assumption that
\[
\sup_{r>0} \sup_{\|\Bw\|_1=1/2} \max(1,r^{2d+4}) |\hat{f}(r\Bw)| \le B,
\] 
we have $|h(\Bw,b)|\le 10(|\hat{f}(\Bw)|+|\hat{f}(-\Bw)|)\le 20 B$ and 
\begin{align*}
&\left|\frac{1}{(2\|\Bw\|_1)^{2d+4}} h \left(\frac{\Bw}{4\|\Bw\|_1^2},\frac{b}{4\|\Bw\|_1^2} \right) \right| \\
\le& 10r^{2d+4} \left(\left|\hat{f}(r \widetilde{\Bw})\right|+\left|\hat{f}(-r \widetilde{\Bw})\right|\right) \le 20 B,
\end{align*}
where $\widetilde{\Bw} = \Bw/(2\|\Bw\|_1)$ and $r= 1/(2\|\Bw\|_1)$. Therefore, $|g(\Bw,b)|\le 40 B$.

\section*{Proof of Lemma \ref{app integral}}\label{Appendix B}

We denote the function $h(\Bx,\Bw):= g(\Bw)\sigma(\Bw \cdot \Bx)$ and the random variables
\begin{align*}
Z:=& \sup_{\Bx\in [-r,r]^d} \left| f(\Bx) - \frac{1}{n} \sum_{i=1}^n h(\Bx,\Bw_i) \right|, \\
Z_+ :=& \sup_{\Bx\in [-r,r]^d} f(\Bx) - \frac{1}{n} \sum_{i=1}^n h(\Bx,\Bw_i),  \\
Z_- :=& \sup_{\Bx\in [-r,r]^d}  \frac{1}{n} \sum_{i=1}^n h(\Bx,\Bw_i) - f(\Bx).
\end{align*}
We are going to bound the expectation of these random variables by the Rademacher complexity \cite{bartlett2002rademacher}. Let $\Bw_{1:n}' = (\Bw_i')_{1\le i \le n}$ be i.i.d. samples from the uniform distribution on $[-1/2,1/2]^d$, independent of $\Bw_{1:n}$. Since $\bE_{\Bw_i'} h(\Bx,\Bw_i') = f(\Bx)$, 
\begin{align*}
&\bE_{\Bw_{1:n}} Z_+ \\
=& \bE_{\Bw_{1:n}} \left[\sup_{\Bx\in [-r,r]^d} \bE_{\Bw_{1:n}'} \frac{1}{n} \sum_{i=1}^n h(\Bx,\Bw_i') - \frac{1}{n} \sum_{i=1}^n h(\Bx,\Bw_i) \right] \\
\le& \bE_{\Bw_{1:n}, \Bw_{1:n}'} \left[ \sup_{\Bx\in [-r,r]^d} \frac{1}{n} \sum_{i=1}^n h(\Bx,\Bw_i') - h(\Bx,\Bw_i) \right].
\end{align*}
Let $\xi_{1:n} = (\xi_i)_{1\le i \le n}$ be a sequence of i.i.d Rademacher random variables, independent of $\Bw_{1:n}$ and $\Bw_{1:n}'$, Then, by symmetrization argument, we have
\begin{align*}
&\bE_{\Bw_{1:n}} Z_+ \\
\le& \bE_{\Bw_{1:n}, \Bw_{1:n}'} \left[ \sup_{\Bx\in [-r,r]^d} \frac{1}{n} \sum_{i=1}^n h(\Bx,\Bw_i') - h(\Bx,\Bw_i) \right] \\
=& \bE_{\Bw_{1:n}, \Bw_{1:n}',\xi_{1:n}} \left[ \sup_{\Bx\in [-r,r]^d} \frac{1}{n} \sum_{i=1}^n \xi_i (h(\Bx,\Bw_i') - h(\Bx,\Bw_i)) \right] \\
\le& \bE_{\Bw_{1:n}, \Bw_{1:n}',\xi_{1:n}} \left[ \sup_{\Bx\in [-r,r]^d} \frac{1}{n} \sum_{i=1}^n \xi_i h(\Bx,\Bw_i') \right. \\
&\qquad \left.+ \sup_{\Bx\in [-r,r]^d} \frac{1}{n} \sum_{i=1}^n -\xi_i h(\Bx,\Bw_i) \right] \\
=& 2 \bE_{\Bw_{1:n},\xi_{1:n}} \left[ \sup_{\Bx\in [-r,r]^d} \frac{1}{n} \sum_{i=1}^n \xi_i h(\Bx,\Bw_i) \right],
\end{align*}
where the last equality is because $\Bw_i$ and $\Bw_i'$ have the same distribution and $\xi_i$ and $-\xi_i$ have the same distribution. 

Observe that, for any $\Bw\in[-1/2,1/2]^d$ and $\Bx,\By\in [-r,r]^d$, we have $|h(\Bx,\Bw)|\le Brd/2$ and
\begin{align*}
|h(\Bx,\Bw) - h(\By,\Bw)| &= |g(\Bw)\sigma(\Bw\cdot \Bx) - g(\Bw)\sigma(\Bw\cdot \By)| \\
&\le B |\Bw \cdot (\Bx-\By)| \le \tfrac{1}{2} B d \|\Bx-\By\|_\infty.
\end{align*}
For any $\epsilon>0$, there exists a set $S_\epsilon = \{\Bx_j \}_{j=1}^N \subseteq [-r,r]^d$ of $N\le (1+r/\epsilon)^d$ points such that $S_\epsilon$ is an $\epsilon$-covering of $[-r,r]^d$: for any $\Bx$, there exists $\Bx_j \in S_\epsilon$ such that $\|\Bx - \Bx_j\|_\infty \le \epsilon$. Hence,
\begin{align*}
&\bE_{\Bw_{1:n}} Z_+ \\
\le& 2\bE_{\Bw_{1:n},\xi_{1:n}} \left[ \sup_{\Bx_j\in S_\epsilon} \frac{1}{n} \sum_{i=1}^n \xi_i h(\Bx_j,\Bw_i) \right.\\
&\qquad \left. + \sup_{\|\Bx-\Bx_j\|_\infty \le \epsilon} \frac{1}{n} \sum_{i=1}^n \xi_i (h(\Bx,\Bw_i) -h(\Bx_j,\Bw_i)) \right] \\
\le& 2\bE_{\Bw_{1:n},\xi_{1:n}} \left[ \sup_{\Bx_j\in S_\epsilon} \frac{1}{n} \sum_{i=1}^n \xi_i h(\Bx_j,\Bw_i) \right] + Bd \epsilon.
\end{align*}
By Massart's lemma \cite[Lemma 26.8]{shalevshwartz2014understanding}, 
\begin{multline*}
\bE_{\xi_{1:n}} \left[ \sup_{\Bx_j\in S_\epsilon} \frac{1}{n} \sum_{i=1}^n \xi_i h(\Bx_j,\Bw_i) \right] \\
\le \frac{Brd \sqrt{2\log N}}{2\sqrt{n}} \le \frac{Brd}{2} \sqrt{\frac{2d \log (1+r/\epsilon)}{n}}.
\end{multline*}
If we choose $\epsilon = r/\sqrt{n}$, then
\begin{align*}
\bE_{\Bw_{1:n}} Z_+ \le& Brd \sqrt{\frac{2d \log (1+\sqrt{n})}{n}} + \frac{Brd}{\sqrt{n}} \\
\le& 2Brd \sqrt{\frac{2d\log(n+1)}{n}}.
\end{align*}

Next, we apply McDiarmid's inequality \cite[Theorem 6.2]{boucheron2013concentration} to the random variable $Z_+$. Observe that, if one of the $\Bw_j$ is replaced by $\Bw_j'$, then the difference is 
\begin{align*}
&\left| \sup_{\Bx\in[-r,r]^d} \left[ f(\Bx) - \frac{1}{n} \sum_{i=1}^n h(\Bx,\Bw_i) \right]  - \sup_{\Bx\in[-r,r]^d} \right. \\
&\quad \left. \left[ f(\Bx) - \frac{1}{n} \sum_{i=1}^n h(\Bx,\Bw_i) + \frac{1}{n}h(\Bx,\Bw_j) - \frac{1}{n} h(\Bx,\Bw_j') \right] \right| \\
\le& \frac{1}{n} \sup_{\Bx\in[-r,r]^d} \left|h(\Bx,\Bw_j) - h(\Bx,\Bw_j') \right| \le \frac{Brd}{n},
\end{align*}
because the inequalities hold for each fixed $\Bx\in[-r,r]^d$ and taking supremum cannot increase the difference. Hence, McDiarmid's inequality implies
\[
\bP(Z_+- \bE Z_+>t ) \le \exp \left(-\frac{2nt^2}{B^2r^2d^2} \right).
\]
For any $\delta\in(0,1)$, if we choose $t= Brd \sqrt{\log(2/\delta)/(2n)}$, then with probability at least $1-\delta/2$, 
\begin{align*}
Z_+ \le& \bE Z_+ +t \\
\le& Brd \left( 2\sqrt{\frac{2d\log(n+1)}{n}} + \sqrt{\frac{\log(2/\delta)}{2n}} \right) =: \cE.
\end{align*}
Applying similar argument to $Z_-$ shows that $Z_-\le \cE$ with probability at least $1-\delta/2$. Hence,
\[
\bP(Z>\cE) \le \bP(Z_+>\cE) + \bP(Z_->\cE) \le \delta. \qedhere
\]

\bibliographystyle{IEEEtran}
\bibliography{IEEEabrv,References}

%
%

\end{document}